\documentclass[twoside]{article}

\usepackage{amsmath,amsfonts,bm}

\def\eqref#1{equation~\ref{#1}}

\def\1{\bm{1}}

\DeclareMathAlphabet{\mathsfit}{\encodingdefault}{\sfdefault}{m}{sl}
\SetMathAlphabet{\mathsfit}{bold}{\encodingdefault}{\sfdefault}{bx}{n}

\DeclareMathOperator*{\argmax}{arg\,max}
\DeclareMathOperator*{\argmin}{arg\,min}

\RequirePackage[loading]{tracefnt}

\usepackage[accepted]{aistats2025}
\usepackage{hyperref}
\usepackage{url}
\usepackage[round]{natbib}

\usepackage[utf8]{inputenc} %
\usepackage[T1]{fontenc}    %
\usepackage{hyperref}       %
\usepackage{url}            %
\usepackage{booktabs}       %
\usepackage{amsfonts}       %
\usepackage{nicefrac}       %
\usepackage{microtype}      %
\usepackage{xcolor}         %
\usepackage{amsmath}
\usepackage{amssymb}
\usepackage{amsthm}
\usepackage{thmtools}
\usepackage{thm-restate}
\usepackage{algorithm}
\usepackage{algpseudocode}
\usepackage{mathtools}
\usepackage{tikz}
\usepackage{MnSymbol,bbding,pifont}
\usepackage{scalerel}
\usepackage{stackengine}
\usepackage{wrapfig}
\usepackage{bbm}
\usepackage{natbib}

\usepackage{etoc}
\usepackage{ulem}
\usepackage{titlesec}
\usepackage{enumitem}

\usetikzlibrary{shapes.geometric}

\definecolor{gold}{rgb}{1.0, 0.84, 0.0}

\newcommand\openbigstar[1][0.7]{%
  \scalerel*{%
    \stackinset{c}{-.125pt}{c}{}{\scalebox{#1}{\color{gold}{$\bigstar$}}}{%
      $\bigstar$}%
  }{\bigstar}
}

\theoremstyle{plain}
\declaretheorem[name=Theorem,numberwithin=section]{theorem}

\newtheorem{definition}[theorem]{Definition}

\definecolor{applegreen}{rgb}{0.55, 0.71, 0.0}
\definecolor{darkgreen}{rgb}{0.0, 0.2, 0.13}
\definecolor{lightpastelpurple}{rgb}{0.69, 0.61, 0.85}

\newcommand{\purpcirc}{{\tikz\draw[fill=lightpastelpurple,draw=black] (0,0) circle (0.12);}}

\newcommand{\X}{\texttt{PEF}}

\usepackage[textsize=tiny]{todonotes}

\definecolor{ao}{rgb}{0.0, 0.5, 0.0}

\definecolor{royalblue}{rgb}{0.25, 0.41, 0.88}
\definecolor{applegreen}{rgb}{0.55, 0.71, 0.0}

\hypersetup{
  colorlinks,
  citecolor={green!50!black},
  linkcolor=royalblue,
  urlcolor=royalblue}

\begin{document}

\runningtitle{Fundamental Limits of Perfect Concept Erasure}

\runningauthor{Fundamental Limits of Perfect Concept Erasure}

\twocolumn[

\aistatstitle{Fundamental Limits of Perfect Concept Erasure}

\aistatsauthor{\hspace{.5cm}Somnath Basu Roy Chowdhury$^1$ \quad Avinava Dubey$^2$ \quad Ahmad Beirami$^3$ \quad Rahul Kidambi$^2$}
\aistatsauthor{Nicholas Monath$^3$ \quad Amr Ahmed$^2$ \quad Snigdha Chaturvedi$^1$}
\vspace{0.2cm}
\aistatsaddress{$^1$UNC Chapel Hill \quad $^2$Google Research \quad  $^3$Google DeepMind \\ 
\texttt{\href{mailto:somnath@cs.unc.edu, snigdha@cs.unc.edu}{\{somnath, snigdha\}@cs.unc.edu}}, 
\texttt{\href{mailto:avinavadubey@google.com, beirami@google.com, rahulkidambi@google.com, nmonath@google.com}{\{avinavadubey, beirami, rahulkidambi, nmonath\}@google.com}}
} 
]

\begin{abstract}
    Concept erasure is the task of erasing information about a concept (e.g., gender or race) from a representation set while retaining the maximum possible utility --- information from original representations. Concept erasure is useful in several applications, such as removing sensitive concepts to achieve fairness and interpreting the impact of specific concepts on a model's performance. Previous concept erasure techniques have prioritized robustly erasing concepts over retaining the utility of the resultant representations. However, there seems to be an inherent tradeoff between erasure and retaining utility, making it unclear how to achieve perfect concept erasure while maintaining high utility. In this paper, we offer a fresh perspective toward solving this problem by quantifying the fundamental limits of concept erasure through an information-theoretic lens.
    Using these results, we investigate constraints on the data distribution and the erasure functions required to achieve the limits of perfect concept erasure. Empirically, we show that the derived erasure functions achieve the optimal theoretical bounds. Additionally, we show that our approach outperforms existing methods on a range of synthetic and real-world datasets using GPT-4 representations.
\end{abstract}

\section{INTRODUCTION}
\label{sec:intro}
Modern large-scale machine learning models~\citep{achiam2023gpt, team2023gemini} trained in a self-supervised manner have showcased an impressive array of capabilities across a range of applications~\citep{ahn2022can, sun2023short, singhal2023publisher}. Owing to their success, in many applications practitioners rely on data representations from large models~\citep{bert, vit, radford2021learning, touvron2023llama} and fine-tune smaller networks on top of these representations. Such representations often encode a diverse set of \textit{concepts}~\citep{hinton1986learning} about the underlying data. For example, concepts associated with a video representation 
can be the geographic location or pose of a person. Concepts can also be more abstract, e.g., text representations may encode writing style~\citep{patel2023learning}, topics~\citep{sarkar2023zero}, or even sensitive information like gender or race of the author~\citep{bolukbasi2016man}. 
Since practitioners cannot control the concepts encoded in representations, machine learning models may learn to rely on such sensitive concepts, thereby inadvertently impacting downstream applications. For example, while developing a resume screening system, an organization may want to ensure that the candidate's gender does not affect the hiring decisions.  This calls for techniques that enable practitioners to erase certain concepts encoded in data representations.

To address these challenges, we focus on \textit{concept erasure}~\citep{inlp, rlace, kram}, i.e., removing information about a concept (e.g., gender) from a representation set.  This is performed by transforming the original representations using an erasure function. The transformed representations should encode minimal information about the erased concept (\textit{privacy}) while retaining the maximal information about the original representations (\textit{utility}). This differs from conventional invariant representation learning~\citep{zhao2019learning, nguyen2021domain, zhao2022fundamental, sadeghi2022characterizing}, which only retains information about a specific task (e.g., hiring). Since concept erasure is task-agnostic, its output representations can be used in a wide range of applications. This is useful when a data distributor wants to share data representations with its customers while safeguarding sensitive concepts, such as the gender or race of the data producer~\citep{zemel2013learning}. Additionally, prior work has shown that concept erasure is useful in improving  fairness~\citep{inlp, farm, belrose2024leace} and interpretability~\citep{gonen2020s, elazar2021amnesic}.

{A long line of work~\citep{bolukbasi2016man, inlp, rlace, ravfogel2022kernelized, farm, kram} has focused on developing effective concept erasure techniques. Most of these approaches prioritize only erasing concept information. However, in concept erasure, we want to achieve perfect erasure while retaining maximum utility.
Motivated by this challenge, we aim to answer the research question: \textit{What is the maximum utility that can be retained while perfectly erasing a concept, and what functions achieve perfect erasure?}}

In this work, we take a step towards answering the above question by deriving the functions that achieve perfect concept erasure. We formalize the concept erasure task using information-theoretic tools, introducing the definitions of perfect and relaxed 
concept erasure. 
We borrow tools from privacy and make a conceptual connection with trade-offs between utility and privacy~\citep{sreekumar2019optimal, bertran2019adversarially, razeghi2020perfect, atashin2021variational} and build on \citep{calmon2015fundamental, du2017principal} that characterize the fundamental limits of perfect privacy. 
In these data settings, we derive the erasure functions that achieve perfect concept erasure, which we call \textit{perfect erasure functions} ({\texttt{PEFs}}). 
In controlled settings, we show that existing approaches either reveal a significant amount of concept information or entirely lose all utility information from the original representations. 
 
 To summarize, our primary contributions are: 
\begin{itemize}[topsep=1pt, leftmargin=*, noitemsep]
    \itemsep1mm
    \item We formalize the task of concept erasure, introducing the definitions of perfect and relaxed concept erasure using information-theoretic tools. 
    \item We hypothesize constraints on the underlying data required to achieve the information-theoretic outer bounds for perfect concept erasure. 
    \item We derive the perfect erasure functions under different data conditions. 
    \item We extensively evaluate {\X} on synthetic and real-world datasets, including GPT-4 representations.
\end{itemize}

\vspace{-5pt}
\section{BACKGROUND}
\vspace{-5pt}

\textbf{Concept Erasure}. The idea of removing information from data representations was motivated by the problem of removing gender information from GloVe embeddings~\citep{bolukbasi2016man}. This task was later termed as \textit{concept erasure}~\citep{inlp, rlace}, which involved removing information about a \textit{concept} (a random variable) from a representation set. Concept erasure differs from conventional approaches to removing information from representations, such as adversarial learning~\citep{zemel2013learning, pmlr-v80-madras18a, elazar2018adversarial, sadeghi2022characterizing, iwasawa2018censoring, ads, xie2017controllable}, which aim to learn invariant representations for a specific task. For concept erasure, the user does not have access to a specific task and should retain the maximum utility from original representations.  

A simplified form of this task is \textit{linear concept erasure}~\citep{inlp, dev-etal-2021-oscar, rlace}, which removes concept information in a manner that prevents a linear adversary (e.g., a linear classifier) from extracting it.
Initial approaches involve projecting representations onto the nullspace of linear concept classifiers, either once \citep{bolukbasi2016man, haghighatkhah-etal-2022-better} or in an iterative manner \citep{inlp}. Subsequent work~\citep{rlace} generalized the iterative nullspace projection objective as a minimax game.
Recently, \citet{belrose2024leace} 
show that a concept is perfectly linearly erased only when different concept groups share the same centroid representation and provide an optimal algorithm to achieve this condition.

Though theoretically sound and effective in practice, linear concept erasure has limitations, as the erased concept can still be retrieved using a non-linear adversary. 
Recent works have focused on performing \textit{general concept erasure}, which tries to protect the concept against an arbitrarily strong adversary. One approach, utilized by \citep{ravfogel2022kernelized, shao2023gold}, involves projecting inputs onto a non-linear space and then applying linear concept erasure techniques. On the other hand, \citet{farm, kram} learn a parameterized erasure function using rate-distortion based objectives. While empirically promising, these approaches cannot theoretically guarantee perfect erasure nor the retention of maximum possible utility.

General concept erasure techniques often showcase a significant information/utility loss~\citep{kram}, which is expected as perfectly erasing a concept may remove other information. An extreme case occurs when a function outputs random representations, which perfectly erases a concept but results in complete information loss~\citep{lowy2022stochastic}. Therefore, we quantify the maximum possible information that can be retained while perfectly erasing a concept. 
In this work, we use mutual information to quantify the utility and privacy during concept erasure~\citep{sankar2013utility}. Mutual information minimization has been used to improve privacy~\citep{duchi2014privacy, bertran2019adversarially} and fairness~\citep{mary2019fairness, lowy2022stochastic} in many applications. However, we do not optimize the mutual information terms directly; instead, we utilize their outer bounds. This problem can be studied using other divergence measures as well. We use mutual information as its outer bounds enable us to derive the analytic form of the erasure functions.

 A long line of work in information theory \citep{du2017principal, sreekumar2019optimal, bertran2019adversarially, atashin2021variational} has focused on achieving an optimal tradeoff between privacy, utility, and function complexity by optimizing the associated mutual information terms. Bottleneck CLUB~\citep{razeghi2023bottlenecks} provides a general framework to unify these approaches. Concept erasure is a special case of this framework, where we focus only on utility and privacy (also called the privacy funnel~\citep{makhdoumi2014information}). \citet{calmon2015fundamental} was the first to study perfect privacy (or perfect erasure), deriving the PIC conditions necessary to achieve it within the privacy funnel framework. \citet{du2017principal} derived the mutual information bounds for the privacy funnel setting, while \citet{hsu2018generalizing} generalized the privacy funnel problem to a broader class of $f$-divergences.  Building on this, \citet{wang2017estimation, wang2019privacy} focused on erasing functions of the data such that those cannot be reconstructed (using MMSE) and derived sharp privacy-utility tradeoff bounds using $\chi^2$-divergence. More recent works~\citep{razeghi2024deep, huang2024efficient}, leveraged variational approximation using deep networks to optimize the privacy funnel. In contrast to these approaches, we do not optimize the mutual information terms. Instead, we build on the theoretical framework presented by \citet{du2017principal}. Using the mutual information outer bounds from \citep{du2017principal}, we derive the data constraints and erasure functions necessary for achieving perfect concept erasure. Next, we discuss the minimum entropy coupling, which we incorporate into the erasure functions.

\textbf{Minimum Entropy Coupling} (MEC). Minimum entropy coupling $\Gamma$ is a joint distribution over $m$ input probability distributions $(p_1, \ldots, p_m)$, with the minimum entropy. In our work, we are only interested in finding a coupling between two distributions, $\Gamma(p, q)$.
\[\Gamma_{\mathrm{min}}(p, q) = \argmin_\Gamma H_\Gamma(p, q) = -\Gamma_{ij} \log \Gamma_{ij},\]
where $\sum_j \Gamma_{ij} = p_i, \sum_i \Gamma_{ij} = q_j$.
It is easy to see that the minimum entropy, $H_{\mathrm{min}}(p, q)$, achieved by MEC, $\Gamma_{\mathrm{min}}(p, q)$, also maximizes the mutual information: $I(p; q) = H(p) + H(q) - H_{\mathrm{min}}(p, q)$, between two distributions $p$ and $q$. 
However, computing the minimum entropy coupling is NP-Hard~\citep{kovavcevic2015entropy}. Later, \citet{kocaoglu2017entropic} introduced a greedy algorithm to efficiently approximate the MEC. Subsequent works~\citep{kocaoglu2017entropicb, rossi2019greedy, compton2022entropic} showed that the approximation produced by the greedy algorithm is tight. Specifically, while computing the MEC between two distributions, the greedy approach produces approximations within 0.53 bits \citep{compton2023minimum}. We use the greedy algorithm for efficiently computing the MEC to derive the perfect erasure functions. 
In the following sections, we will formalize concept erasure, quantify the outer bounds, and derive perfect erasure functions.

\section{CONCEPT ERASURE}
\label{sec:fce}
\subsection{Problem Statement}
\label{sec:problem}
\vspace{-5pt}
Concept erasure tries to erase the effect of a concept, $A$ (a random variable with support $\mathcal{A}$), from a representation set, $X$ (random variable with support $\mathcal{X}$). 
This is usually done by learning a transformation of the original representations, $Z = f(X)$ (random variable with support $\mathcal{Z}$). We will refer to the transformation map, $f$, as the \textit{erasure function}. In our setup, the erasure function $f$ \textit{only} has access to $X$ and not $A$. This allows $f$ to generalize to examples where the concept $A$ is not available. Note that although prior works have utilized various erasure techniques, they can all be generalized using an erasure function, $f$. For example, linear concept erasure approaches~\citep{bolukbasi2016man, inlp} use a product of projection matrices as $f(x)=\prod_j \mathbf P_jx$, kernelized linear erasure~\citep{ravfogel2022kernelized} use  $f(x)=\prod_j \mathbf P_j \Phi(x)$ (where $\Phi(\cdot)$ is non-linear), \citet{farm, kram, huang2024efficient} explicitly parameterize $f$ using neural networks.

\textbf{Assumptions}. 
We make the following assumptions in our concept erasure setup:\vspace{-5pt}
\begin{enumerate}[topsep=0pt, leftmargin=*, noitemsep, label={({A\arabic*})}]
    \itemsep1mm
    \item The Markov property: $A \longrightarrow X \xlongrightarrow{f} Z$, which implies $I(Z; A|X) = 0$.\label{item:A1} 
    \item The support of the input and output representation sets, $\mathcal{X}$ and $\mathcal{Z}$, are finite.\label{item:A2} 
    \item The support of the concept set is finite, $\mathcal A = \{a_1, a_2, \ldots, a_{|\mathcal{A}|}\}$.
    \label{item:A3} 
    \item {The representations for each concept group are sampled from the distribution, $\forall i \in \{1, \ldots, |\mathcal{A}|\},\; P_i = P({X|A=a_i})$. For all pairs of distributions $(P_i, P_j)$ with $i\neq j$, their supports (defined as $\mathcal{X}_i = \mathrm{supp}(P_i)$) are disjoint, i.e., $\mathcal{X}_i \cap \mathcal{X}_j = \phi$. Similarly, the output representations for each concept group is $P(Z|A=a_i)$.} \label{item:A4} 
    \item The size of the input representation support is larger than the size of concept support, $|\mathcal{X}| > |\mathcal{A}|$. \label{item:A5} 
\end{enumerate}

An example of such a setup is erasing gender information from English word embeddings. The representation set is finite since the number of English words is finite. Therefore, even though data representations are typically continuous, assuming a categorical distribution with a large support set is reasonable.

\textbf{Objective}. The primary objective of concept erasure is to construct an erasure function, $f$, such that these two properties: (a) effectively erase a concept, $A$, achieving $I(Z; A) \approx 0$ and (b) retaining information about the original representations, $X$, achieving high $I(Z; X)$. The resultant objective is (for a small constant $\epsilon>0$):
\begin{equation}
    \max_f I(Z; X) \text{ subject to } I(Z; A) \leq \epsilon,
    \label{eq:ce-obj}
\end{equation}
where $Z=f(X)$. The concept erasure objective shown above is reminiscent of the information bottleneck objective~\citep{tishby2000information}, which focuses on achieving a good tradeoff between $I(Z; A)$ and $I(Z; X)$. In concept erasure, the objective is to robustly erase concept $A$, resulting in near-zero $I(Z; A) \approx 0$, even if it is at the expense of a reduced utility, $I(Z; X)$.

\subsection{Types of Concept Erasure}
We note that optimizing the objective in Eq.~\ref{eq:ce-obj} is difficult, as estimating mutual information using a finite number of samples is challenging~\citep{song2019understanding, mcallester2020formal}. 
{Rather than optimizing the mutual information terms, we use their outer bounds to derive the erasure functions.} 
Next, we introduce several definitions involving concept erasure.
\begin{definition}[Perfect Concept Erasure]
    An erasure function, $f$, achieves \textit{perfect concept erasure} if $I(Z; A)=0$, where $Z=f(X)$. 
    \label{def:pce}
\end{definition}
\vspace{-5pt}
Perfect concept erasure ensures that no information about the concept is present in the resultant representations, $I(Z; A)=0$. However, achieving perfect erasure may not always be feasible. Our next result provides the conditions under which perfect concept erasure is feasible. For deriving the conditions of perfect erasure, we rely on the notion of principal inertia components (PICs)~\citep{du2017principal}. Intuitively, PICs capture the correlation between two random variables $X$ and $A$, and can be viewed as a decomposition of the joint distribution, $p_{A, X}$ (see details in Appendix~\ref{sec:PIC}). 
Note that all of the results in this paper are derived under Assumptions~\ref{item:A1}-\ref{item:A5} unless otherwise stated.
\begin{restatable}[Perfect Concept Erasure Achievability]{lemma}{feasibility}
Under the assumptions~\ref{item:A1}-\ref{item:A4}, for a joint distribution $p_{A, X}$ with finite support defined over $\mathcal{A} \times \mathcal{X}$, perfect concept erasure, $I(Z; A)=0$, is achievable if either (a) smallest principal inertia component, $\lambda_d(A, X) = 0$ or (b) $|\mathcal{X}|>|\mathcal{A}|$.
\label{lem:feasibility}
\end{restatable}
This lemma states that perfect erasure is feasible when either the smallest PIC is zero, $\lambda_d(A, X)=0$, or $\mathcal{|X|} > \mathcal{|A|}$ (proof in Appendix~\ref{sec:PIC}). Throughout the paper, we assume~\ref{item:A5}: $\mathcal{|X|} > \mathcal{|A|}$, which means perfect erasure is always feasible irrespective of $\lambda_d(A, X)$.   This is a reasonable assumption in practice as the size of the representation space is typically much larger than the number of concepts.

Note that Definition~\ref{def:pce} for perfect concept erasure only requires the concept to be erased from learned representations, $I(Z; A)=0$. It does not impose any constraints on the utility of the learned representations, $I(Z; X)$. 
Since concept erasure also involves retaining maximum utility, it is natural to question the maximum achievable utility $I(Z; X)$ during perfect erasure. 
\begin{restatable}[Perfect Erasure Bound]{lemma}{miresult}
Under the assumptions~\ref{item:A1}-\ref{item:A5}, if an erasure function $f: \mathcal{X} \rightarrow \mathcal{Z}$ achieves perfect concept erasure $I(Z; A)=0$, then the utility $I(Z; X)$ is bounded as:
\begin{equation}
    I(Z; X) = I(Z; X|A) \leq H(X|A),\nonumber
\end{equation}
where the equality is satisfied when $H(X|Z, A)=0$. 
\label{lem:mi_result}
\end{restatable}
The above result presents the outer bound for the utility, $I(Z; X)$,  during perfect concept erasure. The complete proof is in Appendix \ref{sec:lem1_proof}. 
Next, we focus on a slightly different setting of concept erasure where the privacy leakage about the concept can be non-zero, $I(Z; A) = \epsilon$  ($\epsilon > 0$). 
Specifically, we are interested in understanding the best achievable privacy (minimum $I(Z; A)$) for any utility $u$, such that $I(Z; X) \geq u$. To address this question, we present the notion of an \textit{erasure funnel}\footnote{This is often termed as privacy funnel, but we will refer to it as erasure funnel in the context of concept erasure.} introduced in information theory literature.
\begin{definition}[Erasure Funnel~\citep{du2017principal}]
For $0 \leq u \leq H(X)$, representations $Z = f(X)$, and a fixed joint distribution over $\mathcal{A} \times \mathcal{X}$, we define the \textit{erasure funnel} function as:
\begin{equation}
    \epsilon(u) = \inf_f \{I(Z; A)| I(Z; X) \geq u, A \rightarrow X \xrightarrow{f} Z\}.
    \label{eqn:eu}
\end{equation}
\label{def:ef}
\end{definition}
\vspace{-15pt}
We use the erasure funnel to define the setting where concept leakage is non-zero, $\epsilon(u) > 0$.
We term this setting \textit{relaxed concept erasure}.
\begin{definition}[Relaxed Concept Erasure]
    An erasure function $f$ performs \textit{relaxed concept erasure} for $\epsilon(u) > 0$, if the utility is $I(Z; X) \geq u$ and $I(Z; A) = \epsilon(u)$, where $\epsilon(u) = \inf_f \{I(Z; A)| I(Z; X) \geq u, A \rightarrow X \xrightarrow{f} Z\}$.
    \label{def:relaxed}
\end{definition}
\vspace{-5pt}
In Definition~\ref{def:relaxed}, $\epsilon(u)$ is the minimum achievable privacy leakage, $I(Z; A)$, such that the utility is at least $I(Z; X) \geq u$. This definition considers the utility, $I(Z; X)$,  and focuses on finding an erasure function $f$ that does not erase all information. This is useful in applications such as fairness, where it is not necessary to completely erase concepts to achieve fair outcomes.

\subsection{Outer Bounds For Concept Erasure}
We build on the definitions introduced in the previous section and present the outer bounds for privacy. Ideally, we want to \textit{minimize} privacy leakage, $I(Z; A)$ while \textit{maximizing} utility, $I(Z; X)$.  Here, we utilize the erasure funnel (Def.~\ref{def:ef}) to answer these questions. The erasure funnel extends beyond relaxed erasure and can offer insights into perfect erasure ($\epsilon(u) = 0$).  
We will use $\epsilon(u)$ to answer the following research questions: 
\begin{itemize}[topsep=0pt, leftmargin=11mm, noitemsep]
    \itemsep0.25mm
    \item[(\textbf{RQ1})] {What is the maximum utility  $u$ during perfect concept erasure (when $\epsilon=0$)?} 
    \item[(\textbf{RQ2})] {What is the outer bound of privacy $\epsilon(u)$ (minimum $\epsilon(u)$) during relaxed concept erasure?} 
\end{itemize}
\begin{figure}[t!]
    \centering
    \includegraphics[width=0.48\textwidth, keepaspectratio]{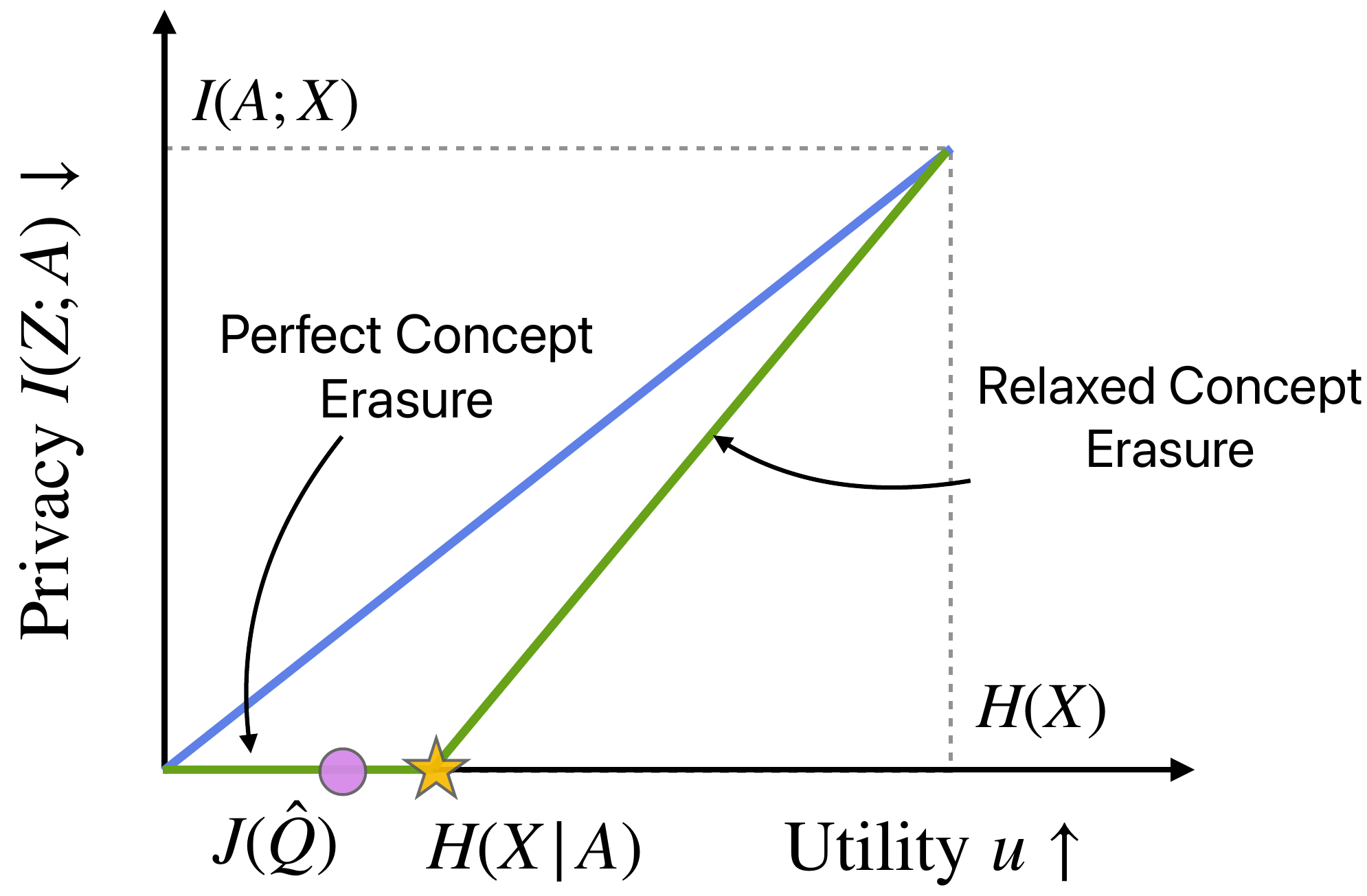}
    \vspace{-10pt}
    \caption{Schematic diagram of the privacy-utility tradeoff and illustration of the bounds of erasure funnel function, $\epsilon(u)$~\citep{du2017principal}. The region shown by the \textcolor{ao}{\textbf{green}} line is desirable for concept erasure as it is the minimum $I(Z; A)$ for a given utility $I(Z; X) \geq u$. 
    } 
    \label{fig:funnel}
    \vspace{-10pt}
\end{figure}
As obtaining the precise analytic expression for $\epsilon(u)$ is difficult, we present its bounds~\citep{du2017principal} to answer the above research questions.
\begin{restatable}[Outer Bounds for $\epsilon(u)$~\citep{du2017principal}]{lemma}{funnel}
For $0 \leq u \leq H(X)$, erasure funnel function, $\epsilon(u)$, is bounded as:
\begin{equation}
    \max\left\{0, u - H(X|A)\right\} \leq \epsilon(u) \leq \frac{uI(A; X)}{H(X)}. \nonumber
\end{equation}
\label{lem:funnel}
\end{restatable}
\vspace{-15pt}
In Figure~\ref{fig:funnel}, we plot the result of Lemma~\ref{lem:funnel}, where the $y$-axis indicates the privacy 
$I(Z; A)$, and the $x$-axis indicates the utility, $u$.
We observe that $\epsilon(u)$ lies within a funnel-like structure with upper bound shown in \textcolor{royalblue}{\textbf{blue}} and lower bound in \textcolor{ao}{\textbf{green}}. 
In this plot, perfect concept erasure is the flat portion of the \textcolor{ao}{\textbf{green}} line where $\epsilon(u) = 0$. 
During perfect erasure, we observe that the maximum possible utility is $H(X|A)$ shown by {\openbigstar[0.8]}. This is the same result we derived in Lemma~\ref{lem:mi_result} and provides the solution for (\textbf{RQ1}). During relaxed concept erasure $\epsilon(u) > 0$, the minimum mutual information is $\epsilon(u) = \frac{uI(A; X)}{H(X)}$, which is a linear function of the utility, $u$. This provides the solution for (\textbf{RQ2}). 
We will discuss the importance of {\purpcirc} in Section~\ref{sec:erase_func}.

We would like to emphasize that these outer bounds might not be achievable for any random variables  
$X$ and $A$. Therefore, our next focus will be on understanding the data constraints under which the outer bounds can be attained and identifying the erasure functions that achieve them. In this work, we only consider perfect concept erasure and its corresponding erasure functions. We defer the derivation of erasure functions that achieve relaxed concept erasure to future works.

\begin{figure}[t!]
    \centering
    \includegraphics[width=0.41\textwidth]{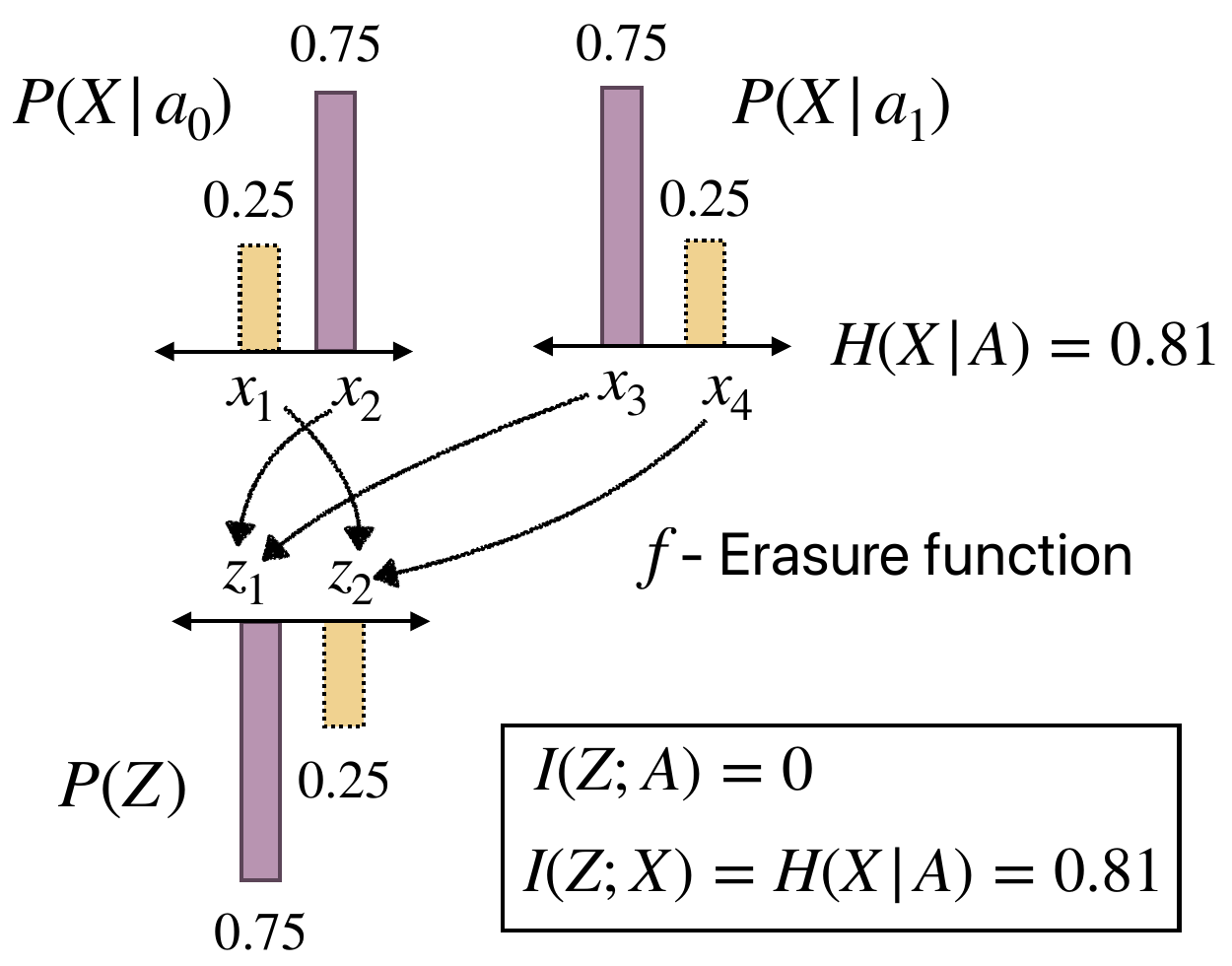}
    \vspace{-10pt}
    \caption{{An illustration of a perfect erasure function. The input representations of each group are sampled from 2D distributions. The erasure function $f$ in the figure achieves the outer bounds for concept erasure.}}
    \vspace{-10pt}
    \label{fig:toy}
    \vspace{-5pt}
\end{figure}

\section{PERFECT ERASURE FUNCTIONS (\X)}
\label{sec:erase_func}

In this section, we present
an algorithm for concept erasure, which is inspired by achieving perfect concept erasure, $I(Z; A) = 0$, along with the utility outer bound, $I(Z; X)=H(X|A)$ ({\openbigstar[0.8]} in Fig.~\ref{fig:funnel}). We will refer to these methods as \textit{perfect erasure functions}.

To motivate our method, we consider a simple example in Fig.~\ref{fig:toy}. 
We consider a scenario with two concept groups, $\mathcal{A} = \{a_0, a_1\}$ with equal prior $p(a_0) = p(a_1) = 0.5$. We consider the data representation space $\mathcal{X}$ to be $\{x_1,x_2,x_3,x_4\}$. Each group has a two dimensional representation -- $P(X|A=a_0)$ has support of $\{x_1, x_2\}$ and $P(X|A=a_1)$ has a support of $\{x_3,x_4\}$.

In this example, there is an erasure function, $f$, which maps the data $X$ to the erased representations, $Z$, which makes use of a helpful observation -- $P(X|A=a_0)$ and $P(X|A=a_1)$ are permutations of one another. This erasure function meets the definition of a perfect erasure function.
We hypothesize that the fact that the conditional distributions are permutations is essential to the achievement of perfect erasure. This permutation is essential as we can have a bijective map that transforms $X$ into $Z$ without revealing $A$.
We formalize this in the following class of data.

\begin{restatable}[Data Constraints]{definition}{datacond} We consider 
the class of data for which the distributions of different groups are permutations of each other:\\[-1em]
\begin{equation*}
     \forall (i, j), \text{ (a) }  |\mathcal{X}_i| = |\mathcal{X}_j|,\text{ (b) } \forall x \in \mathcal{X}_i,\; P_i(x)=P_j(\sigma_{ij}(x)),
\end{equation*}
where $P_i=P(X|A=a_i)$, and for which there exists an erasure function $f$ that achieves $I(Z; A) = 0$ and $I(X; Z) = H(X|A)$.
\label{thm:data-cond}
\end{restatable}
The above definition states that the distributions must be \textit{equal up to permutation} and have equal support sizes ($\mathcal{X}_i=\text{supp}(P_i)$). We provide the justification behind the choice of these data constraints in Appendix~\ref{sec:data-cond}.
Next, we turn to answer the question of how to find such an erasure function $f$ to achieve this definition. 

\textbf{Equal Distributions} (\textit{up to permutations}). 
In the setting, where the distributions $P_i$'s are equal up to permutation, we present the erasure functions. Intuitively, the erasure function is a one-to-one map from each distribution, $P_i$, to a common distribution, $Q=P(Z)$ (as shown in Figure~\ref{fig:toy}). The common distribution $Q$ can be any permutation of the $P_i$'s. We formalize this in the following method.
\begin{restatable}[Perfect Erasure Function]{method}{equalpef} 
    If the data constraints in Definition~\ref{thm:data-cond} hold, then define   erasure function $f$ as:\\[-1em]
    \begin{equation}
        f(x) = \{\sigma_i(x) |\; x \in \mathcal{X}_i, \; \sigma_i \in \Sigma_i\},
        \label{eqn:pef}
    \end{equation}
    where $\sigma_i$ is any bijective function that transforms $P_i$ into $Q$ defined by the set, $\Sigma_i$ below:\\[-0.5em]
    \begin{equation}
        \Sigma_i = \left\{\sigma: \mathcal{X}_i \rightarrow \mathcal{Z}| \forall x \in \mathcal{X}_i,\; P_i(x) = Q(\sigma(x))\right\}.
        \label{eqn:perm_set}
    \end{equation}
    \label{lem:equal-pef}%
\end{restatable}
\vspace{-15pt}
The above method shows that a piecewise bijective mapping different concept groups, $\mathcal{X}_i$, to a common output support, $\mathcal{Z}$, achieves the outer bounds for perfect erasure (Definition~\ref{thm:data-cond}). The justification is presented in Appendix~\ref{sec:equal-pef}. Since the choice of $Q$ and $\sigma_i$'s are not unique, there can be multiple formulations of the erasure function. In the simple scenario, where all $P_i$'s are uniform distributions, then $f$ can be any random bijective map between $\mathcal{X}_i$ and $\mathcal{Z}$.

\textbf{Unequal distributions}. We consider the scenario where the constraints of Definition~\ref{thm:data-cond} are violated, i.e., the distributions, $\{P_i\}$'s, are not equal (even after permutation). In this setting, achieving the outer bounds mentioned in Definition~\ref{thm:data-cond} is not possible. However, it is still possible to achieve perfect concept erasure, $I(Z; A)=0$. Operating with ${P}(Z) = {P}(Z|A=a_i)$, we want to find the piecewise map $f_i: \mathcal{X}_i \rightarrow \mathcal{Z}$ such that the distribution of the common support is fixed and equal to the conditional distribution, $Q = {P}(Z) = {P}(Z|A=a_i)$. Collectively the piecewise components $f_i$ define the erasure function $f$.
Specifically, we seek mapping functions $f_i$'s that maximize the utility, $I(Z; X)$, as shown below:
\begin{align}
    \max\; &I(Z; X) = \max I(Z; X|A) \\
    &= \max \sum_{i} p(a_i)I(Z; X|A=a_i)
    \nonumber\\
    &= \max \sum_{i} p(a_i)\left[H(Q) + H(P_i) - H_{\Gamma_i}(Q, P_i)\right]\nonumber\\
    &= \max \left[H(Q)  - \sum_{i} p(a_i)  H_{\Gamma_i}(Q, P_i)\right]. \label{eqn:obj}
\end{align}
The optimization of the above objective depends on two key sets of parameters: (i) the distribution of the common support, $Q$, and (ii) $\Gamma_i$'s -- coupling joint distributions between pairwise distributions $(P_i, Q)$. The second term in the optimization (Eq.~\ref{eqn:obj}) is equivalent to estimating the minimum entropy coupling, $H_{\Gamma_i}(P_i, Q)$, which can be obtained efficiently using the greedy algorithm (Section~\ref{sec:fce}) if we know $Q$. Since $Q$ is unknown, we can simplify the optimization problem as:
\begin{align}
    Q^* &= \argmax_Q J(Q) \nonumber\\
    &= \argmax_Q \left[H(Q) - \sum_i p(a_i)H_{\mathrm{min}}(P_i, Q)\right].
    \label{eqn:q-obj}
\end{align}
\begin{figure}[t!]
    \centering
    \includegraphics[width=0.49\textwidth, keepaspectratio]{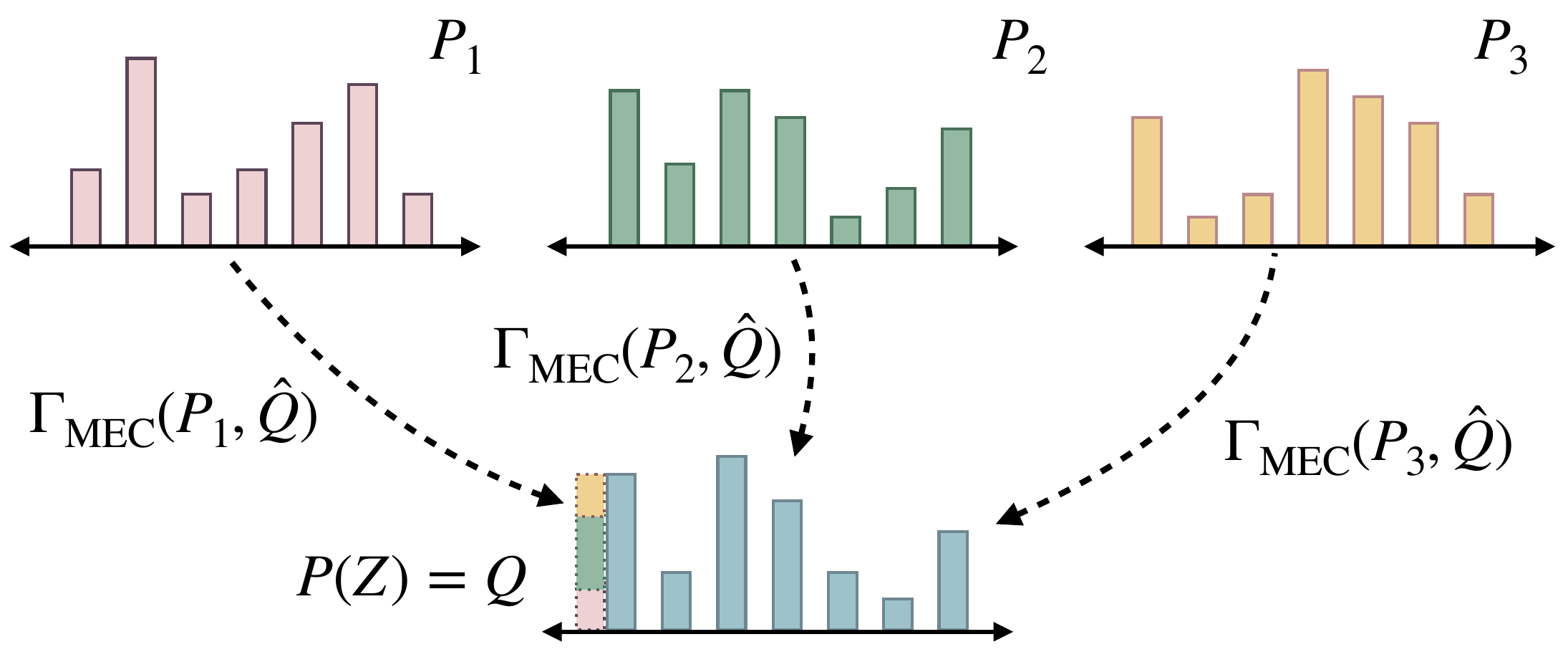}
    \vspace{-10pt}
    \caption{
    Schematic diagram of the proposed concept erasure process for unequal distributions. Given distributions from different concept groups, $\{P_1, P_2, P_3\}$, we map these representations to a common distribution, $Q$. Each distribution $P_i$ is mapped to $Q$ using the minimal entropy coupling map, $\Gamma_{\mathrm{MEC}}(P_i, Q)$. 
    }
    \label{fig:erasure}
\end{figure}
Solving the above objective is difficult and cannot be optimized using standard techniques like gradient descent as we do not know the analytical form of the minimum entropy coupling function, $H_{\mathrm{min}}(\cdot, P_i)$. Even evaluation of the $H_{\mathrm{min}}(\cdot, P_i)$ function is expensive as MEC estimation is NP-Hard and approximating it still requires running the greedy algorithm. Therefore, we resort to Bayesian Optimization (BO)~\citep{frazier2018tutorial}, which performs optimization by making a small number of queries to the objective $J(Q)$ (Eq.~\ref{eqn:q-obj}). Bayesian Optimization utilizes a surrogate model (usually a Gaussian process~\citep{seeger2004gaussian}) to fit the observed query points.
Using an acquisition function, BO explores different areas of the input space to maximize the objective. We will refer to the solution returned by Bayesian Optimization as $\hat{Q}_{\mathrm{BO}}$; however, it is not guaranteed to be the optimal solution. 

We hypothesize that the set of input distributions $\Bar{P}$ are a good candidate for output distribution $Q$ as they are stationary points for the objective $J(Q)$. The best achievable utility, $I(Z; X)$, when  $Q \in \overline{P}$ is shown by the {\purpcirc} point in Figure~\ref{fig:funnel}. 
The set $\bar P$ only represents the set of stationary points and it may be still possible for the global maxima, $Q^*$ to improve upon them. However, when distributions are unequal analytically deriving the global optima $Q^* \notin \Bar{P}$  is non-trivial and remains an open question.\footnote{Further discussion of optimization in Appendix~\ref{appdx:opt}.}

To summarize, when underlying distributions $P_i$'s are unequal, we select either by using the stationary points or the solution of Bayesian optimization: $\hat{Q} = \argmax_{Q \in \Bar{P}\cup \hat{Q}_{\mathrm{BO}}} J(Q)$. In practice, we observe that the stationary points $\Bar{P}$ often achieve better performance than the BO solution, $\hat{Q}_{\mathrm{BO}}$.
Once, we have selected the support distribution $\hat Q$, the erasure function is a stochastic map, $f(x) \sim P(Z|X=x)$, that samples the output $z$ from the conditional distribution:\\[-0.5em]
\begin{equation*}
    P(Z|X=x) = \left\{\Gamma_i(x)| \text{ } x \in \mathcal{X}_i, \Gamma_i = \argmin_\Gamma H_\Gamma(P_i, \hat{Q})\right\},
    \label{eq:erasure-func}
\end{equation*}
where $\Gamma_i$ is the coupling map corresponding to the minimum entropy coupling between distribution pairs $(P_i, \hat{Q})$. An illustration of this erasure function is shown in Figure~\ref{fig:erasure}. Using the result from Lemma~\ref{lem:mi_result}, it is easy to see that for unequal distributions the erasure function doesn't achieve the utility outer bound as $H(X|Z, A) > 0$ (as different $x \in \mathcal{X}_i$ can map to the same $z \in \mathcal{Z}$). 
Nonetheless, since the output distribution $Q$ is fixed, and MEC yields a valid joint distribution, it still ensures perfect concept erasure, $I(Z; X) = 0$. 
\begin{figure*}[t!]
    \centering
    \includegraphics[width=0.89\textwidth, keepaspectratio]{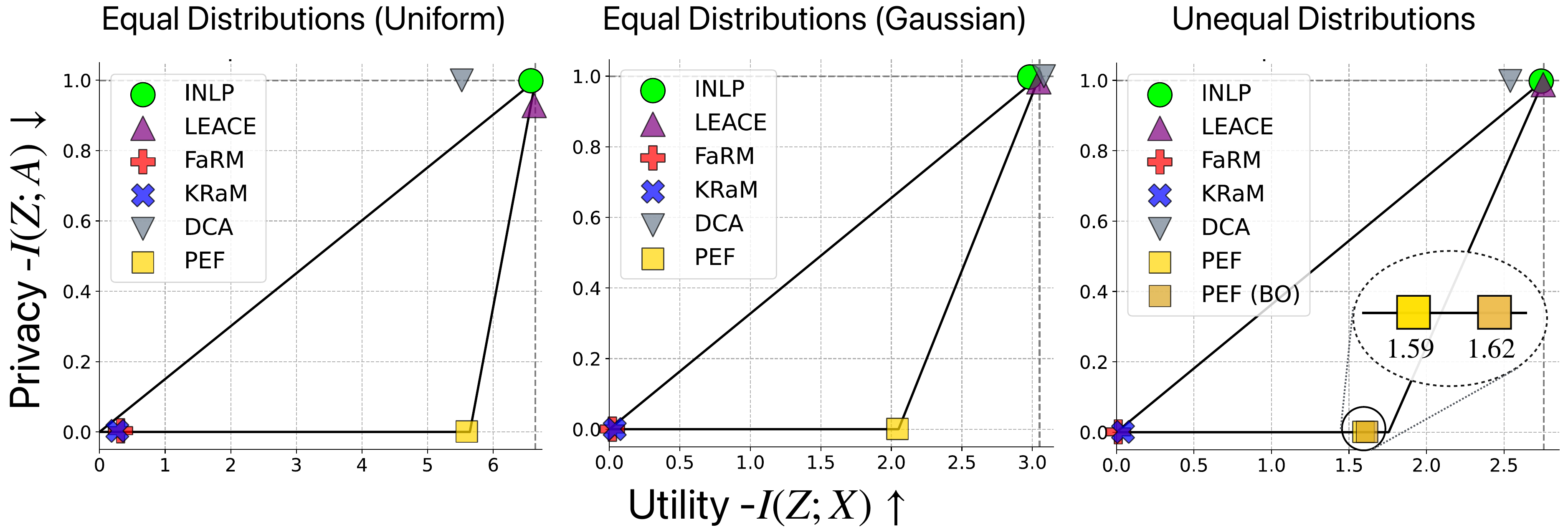}
    \vspace{-12pt}
    \caption{Privacy-Utility tradeoff plots for concept erasure using synthetic data when the underlying representations from different concept groups are sampled from equal uniform (\textit{left}), equal Gaussian (\textit{center}) and unequal (\textit{right}) distributions. We observe that {\X}  retains significant utility (high $I(Z; X)$) with perfect erasure, $I(Z; A) = 0$.}
    \label{fig:syn-results}
\end{figure*}
\begin{figure*}[t!]
    \centering
    \includegraphics[width=0.9\textwidth, keepaspectratio]{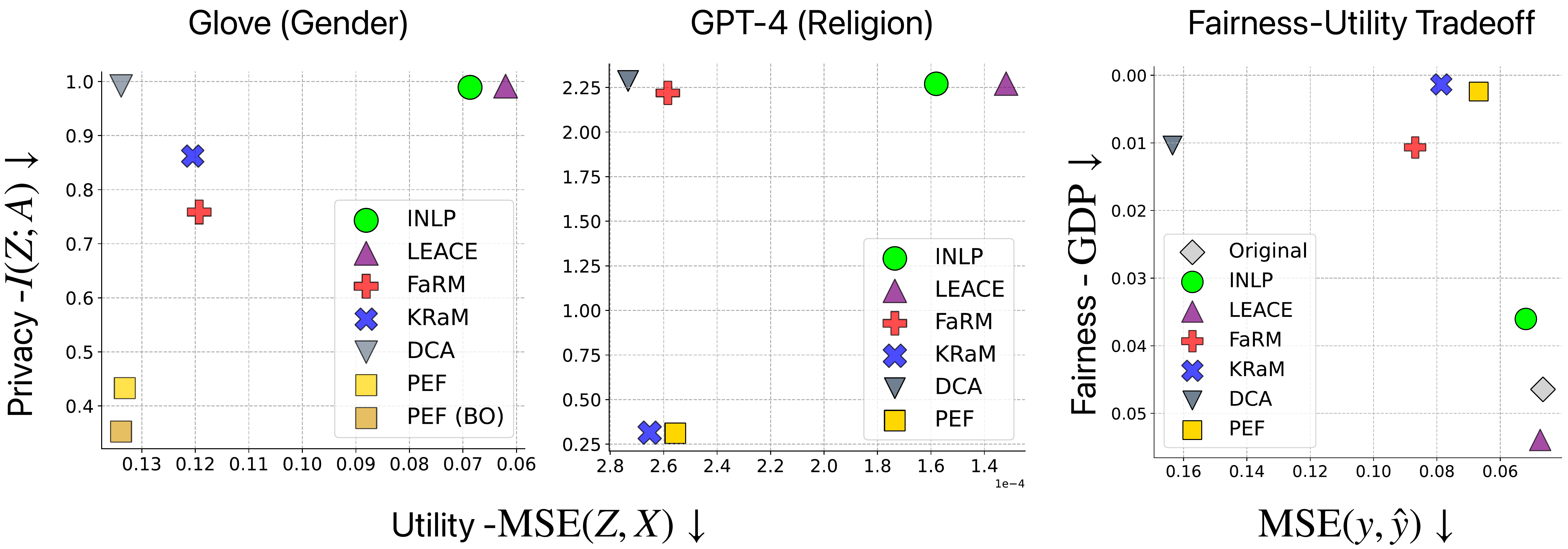}
    \vspace{-10pt}
    \caption{(\textit{Left}) We show the privacy-utility tradeoff for erasing gender information from GloVe embeddings and religion information from GPT-4 embeddings. We observe that {\X} achieves significant improvement in privacy compared to baseline approaches.  (\textit{Right}) We use the erased GPT-4 embeddings for toxicity classification and plot the fairness-utility tradeoff. We observe that {\X} achieves the best fairness-utility tradeoff among the baselines.}
    \label{fig:glove-gpt}
\end{figure*}

We summarize the overall concept erasure algorithm in Algorithm~\ref{alg:erasure}. The algorithm is provided with the entire representation set $\mathcal{S} = \{x_1, x_2, \ldots\}$ (where $x_i \in \mathcal{X}$) and the corresponding concept set, $\mathcal{G}$. First, it estimates the distribution of the representations from individual concept groups. Second, it checks whether the underlying distributions are equal up to permutation. Based on this step, the algorithm determines the optimal erasure function. Finally, the selected erasure function is used to generate the erased representations, $\mathcal{O}$.
\begin{algorithm}[h]
\caption{Concept Erasure Using {\X}}
\begin{algorithmic}[1]
    \State \textbf{Input}: Sample Set $\mathcal{S} = \{x_1, x_2, \ldots \}$, Concept Set $\mathcal{G} = \{a_1, a_2, \ldots\}$. 
    \State $\forall i, P_i = \mathrm{EstimateDistribution}(\mathcal{S}_i)$ \textcolor{gray}{// ($\mathcal{S}_i = \{x|x \in \mathcal{X}_i\}$ is the sample set of the $i$-th concept group)} 
    \If{$\forall (i, j), P_i = \sigma(P_j)$} 
    \State \textcolor{gray}{// Distributions are equal up to permutation}
    \State $f(x) = \{\sigma_i(x) |\; x \in \mathcal{X}_i, \; \sigma_i \in \Sigma_i\}$ \textcolor{gray}{// ($\Sigma_i$ is defined in Eq.~\ref{eqn:perm_set})}
    \State $\mathcal{O} = f(\mathcal{S})$ 
    \Else 
    \State \textcolor{gray}{// Unequal distributions}
    \State $\hat{Q} = \argmax_Q \left[H(Q) - \sum_i p(a_i)H_{\mathrm{min}}(P_i, Q)\right]$
    \State {$P(Z|X=x) =\left\{\Gamma_i(x)| \text{ } x \in \mathcal{X}_i, \Gamma_i = \argmin_\Gamma H_\Gamma(P_i, \hat{Q})\right\}$}
    \State $\mathcal{O} = \{z \sim P(Z|X=x)|\; x \in \mathcal{X}\}$ 
    \EndIf
    \State \Return $\mathcal{O}$ \textcolor{gray}{// erased representation set}
\end{algorithmic}
	\label{alg:erasure}
\end{algorithm}

\section{EXPERIMENTS}
\label{sec:exp}
\vspace{-5pt}
In this section, 
we discuss the experimental setup and evaluation results in detail. The implementation is available here: \href{https://github.com/brcsomnath/PEF}{https://github.com/brcsomnath/PEF}.

\textbf{Baselines}. We compare our approach, {\X}, with several state-of-the-art baselines, including linear concept erasure techniques -- INLP~\citep{inlp}, LEACE~\citep{belrose2024leace} and general erasure techniques -- FaRM~\citep{farm}, KRaM~\citep{kram}. We also compare with the state-of-the-art privacy funnel solver that directly optimizes mutual information terms using variational approximation, DCA~\citep{huang2024efficient}.

\textbf{Synthetic Data}.
We conduct experiments in controlled settings using synthetic data, which allows us to directly compare {\X}'s performance with the theoretical bounds (shown in Figure~\ref{fig:funnel}). Otherwise, accurately estimating mutual information (MI) using high-dimensional data is challenging.
Specifically, we consider a set of 3D representations sampled from a finite support set. The finite support set is also randomly sampled from a 3D uniform distribution. We use two underlying concepts and sample representations from distinct support for each group.  We experiment in two settings where the groups have  (i) \textit{equal} and (ii) \textit{unequal} support distributions to simulate the scenarios in Section~\ref{sec:erase_func}. Finally, we report the MI estimates using $\mathcal{V}$-information~\citep{xu2019theory}, which uses a classifier to predict the concept (privacy) and original representation (utility) from the representations post-concept erasure. Since the support of the input representation set is finite, we associate each sample representation with a categorical label. The classifier, used to compute the $\mathcal{V}$-information with the original representations, predicts the label given the erased representation. $\mathcal{V}$-information provides a strong lower bound for the MI estimate when the classifier is constrained by computational resources.

\begin{figure}[t!]
    \centering
    \includegraphics[width=0.48\textwidth, keepaspectratio]{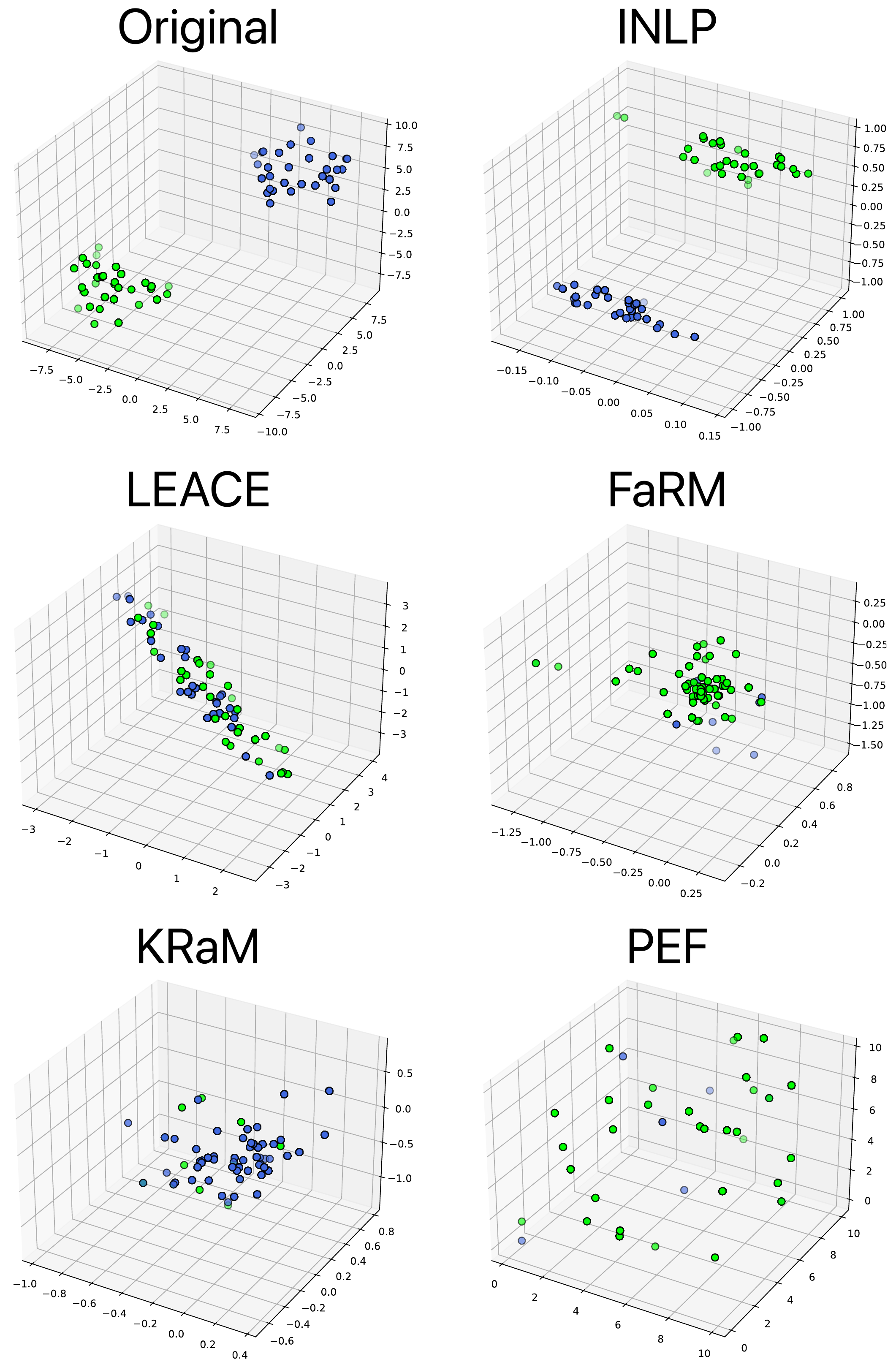}
    \vspace{-20pt}
    \caption{Visualization of the representations post concept erasure using different methods. We observe that existing techniques either leak concept information (INLP \& LEACE) or lose utility as representations collapse (FaRM \& KRaM). {\X} can erase concepts properly while maintaining utility.}
    \label{fig:erasure-vis}
\end{figure}

In Figure~\ref{fig:syn-results}, we report the privacy, $I(Z; A)$, and utility, $I(Z; X)$, tradeoff during concept erasure. The lines describing the funnel structure indicate the theoretical erasure funnel bounds, $\epsilon(u)$, as described in Figure~\ref{fig:funnel}. We observe that linear concept erasure approaches, INLP and LEACE, retain significant utility, achieving high $I(Z; X)$ but at the cost of poor privacy, high $I(Z; A)$. DCA performs similar to linear erasure approaches. On the other hand, nonlinear erasure approaches such as FaRM and KRaM erase the concept robustly, but result in a significant drop in utility, low $I(Z; X)$. In contrast to these approaches, {\X} achieves high utility while perfectly erasing the concept, $I(Z; A)=0$. When representations are sampled from equal distributions in Figure~\ref{fig:syn-results} (left {\& center}), we observe that {\X} achieves the utility outer bound. 
When the distributions are unequal, in Figure~\ref{fig:syn-results} (right), {\X} achieves perfect privacy but is unable to achieve the maximum utility, which is expected.  In this setting, the Bayesian optimization solution, {\X} (BO), slightly improves the utility while achieving perfect erasure.

\textbf{Real-World Datasets}. 
We evaluate our approach in real-world settings to erase gender from GloVe embeddings and religion from the Jigsaw~\citep{jigsaw} dataset's GPT-4 text embeddings.
As these representations are high-dimensional, computing the mutual information with the original representations is challenging. 
Therefore, we report the mean squared error (MSE) in predicting the original representations from the erased representations using a scikit-learn~\citep{pedregosa2011scikit} neural network regression model.

For Glove embeddings shown in Figure~\ref{fig:glove-gpt} (left), we observe that {\X} achieves significantly better privacy (low $I(Z; A)$) 
compared to all baselines while achieving similar utility to methods like FaRM, KRaM, and DCA. We also observe that {\X} (BO) slightly outperforms {\X} in terms of privacy. 
{\X} assumes access to the original probability distributions. In practice, we have finite high-dimensional samples and use kernel density estimation~\citep{chen2017tutorial} to estimate the distributions.  We hypothesize that the slight drop in utility is due to the unavailability of these original distributions.

In the second setting, we obtain GPT-4 (\texttt{text-embedding-3-large}) representations for online comments in the Jigsaw dataset. The objective of this experiment is to erase religion information and utilize the resultant representations for toxicity classification. This ensures that religion does not affect the predicted toxicity scores. We report the results in Figure~\ref{fig:glove-gpt} (center). Similar to the previous setting, we observe significant privacy improvements with similar utility to non-linear erasure techniques (FaRM, KRaM \& DCA). In this setting, we found that the local minima $\{P_i\}$' provided a better solution than $\hat{Q}_{\mathrm{BO}}$, therefore, the performance {\X} and {\X} (BO) is the same. This is expected as Bayesian optimization is not known to be effective in high dimensions (3072).  

After concept erasure, we use the resultant representations to perform toxicity classification. We report the fairness-utility tradeoff, where fairness is computed using generalized demographic parity~\citep{gdp} (for continuous-valued toxicity scores) and utility is the MSE loss of the predictions. In Figure~\ref{fig:glove-gpt} (right), we observe that {\X} significantly improves fairness GDP scores while having minimal impact on performance (minimal increase in MSE).
In general, {\X} achieves a good balance between utility and fairness compared to the baselines. We provide more details in Appendix~\ref{sec:addl_exp}.

\noindent\textbf{Visualization}. In this section, we visualize the representations obtained from various methods, as shown in Figure~\ref{fig:erasure-vis}. We use synthetic 3D representations, which encode a binary concept shown by representations in blue and green. Post concept erasure, we expect representations from both groups to be indistinguishable. For linear erasure techniques (INLP \& LEACE), we observe that the underlying concept can be identified, while non-linear techniques (FaRM \& KRaM) achieve better overlap but collapse into a small region of space, signifying a utility loss. In contrast to these approaches, {\X} shows robust erasure while retaining utility (no collapse observed). We further illustrate the erasure process of LEACE and {\X} with a simpler example in Appendix~\ref{sec:leace-vis}.

\section{CONCLUSION}
In this paper, we study the fundamental limits of perfect concept erasure. Theoretically, we study the maximum amount of information that can be retained while perfectly erasing a concept. Furthermore, we explore the conditions that the underlying data distribution and the analytical form of the erasure function required to achieve perfect erasure. Empirically, we demonstrate the effectiveness of the proposed perfect erasure function (\X) in various synthetic and real-world settings. Although {\X} can guarantee perfect erasure, it assumes access to a finite support representation distribution, which may be difficult to estimate in low-resource settings. Future works can focus on effectively estimating the underlying distribution using a small number of samples to improve the performance of {\X}. 

\section*{Acknowledgment}
The authors are thankful to Flavio du pin Calmon for providing helpful pointers to related work. This work was supported in part by NSF grant DRL-2112635.

\bibliography{aistats}
\bibliographystyle{plainnat}

\clearpage
\section*{Checklist}

 \begin{enumerate}

 \item For all models and algorithms presented, check if you include:
 \begin{enumerate}
   \item A clear description of the mathematical setting, assumptions, algorithm, and/or model. [Yes]
   \item An analysis of the properties and complexity (time, space, sample size) of any algorithm. [Yes]
   \item (Optional) Anonymized source code, with specification of all dependencies, including external libraries. [Yes]
 \end{enumerate}

 \item For any theoretical claim, check if you include:
 \begin{enumerate}
   \item Statements of the full set of assumptions of all theoretical results. [Yes]
   \item Complete proofs of all theoretical results. [Yes]
   \item Clear explanations of any assumptions. [Yes]     
 \end{enumerate}

 \item For all figures and tables that present empirical results, check if you include:
 \begin{enumerate}
   \item The code, data, and instructions needed to reproduce the main experimental results (either in the supplemental material or as a URL). [Yes]
   \item All the training details (e.g., data splits, hyperparameters, how they were chosen). [Yes]
         \item A clear definition of the specific measure or statistics and error bars (e.g., with respect to the random seed after running experiments multiple times). [Yes]
         \item A description of the computing infrastructure used. (e.g., type of GPUs, internal cluster, or cloud provider). [Yes]
 \end{enumerate}

 \item If you are using existing assets (e.g., code, data, models) or curating/releasing new assets, check if you include:
 \begin{enumerate}
   \item Citations of the creator If your work uses existing assets. [Yes]
   \item The license information of the assets, if applicable. [Yes]
   \item New assets either in the supplemental material or as a URL, if applicable. [Not Applicable]
   \item Information about consent from data providers/curators. [Not Applicable]
   \item Discussion of sensible content if applicable, e.g., personally identifiable information or offensive content. [Not Applicable]
 \end{enumerate}

 \item If you used crowdsourcing or conducted research with human subjects, check if you include:
 \begin{enumerate}
   \item The full text of instructions given to participants and screenshots. [Not Applicable]
   \item Descriptions of potential participant risks, with links to Institutional Review Board (IRB) approvals if applicable. [Not Applicable]
   \item The estimated hourly wage paid to participants and the total amount spent on participant compensation. [Not Applicable]
 \end{enumerate}

 \end{enumerate}

\newpage
\appendix
\onecolumn

\section{MATHEMATICAL PROOFS}
\label{sec:proofs}
\localtableofcontents

\subsection{Proof of Lemma~\ref{lem:feasibility}}
\label{sec:PIC}
In this section, we present the necessary and sufficient conditions required for achieving perfect concept erasure. First, we introduce the notion of principal inertia components, which will be instrumental in deriving the conditions for perfect erasure.

\textbf{Principal Inertia Components (PICs)}. PICs can be viewed as a decomposition of the joint distribution, $p_{X, A}$, between two random variables $X$ and $A$ (a branch of study dedicated to this is called \textit{correspondence analysis}~\citep{greenacre2017correspondence}). Principal inertia components exhibit interesting properties and are often used to define common measures like maximum correlation, $\rho_m(X, A)$, and Chi-squared correlation,
$\chi^2(X, A)$.

\begin{definition}[Principal Inertia Components~\citep{du2017principal}]
    Let random variables $X$ and $A$ have finite support sets $\mathcal X$ and $\mathcal A$ respectively, and joint distribution $p_{X,A}$. In addition, let $f_0 : X \rightarrow \mathbb{R}$ and $g_0 : Y \rightarrow \mathbb{R}$ be the constant functions $f_0(x) = 1$ and $g_0(a) = 1$. For $k \in \mathbb{Z}^+$, we (recursively) define:
    \begin{align*}
    \lambda_k(X; A ) = \max &\Big\{\mathbb{E}[f(X)g(A)]^2 \, \Big| \, \mathbb{E}[f(X)f_j(X)] = 0, \\
    &\phantom{=} \mathbb{E}[g(Y)g_j(A)] = 0, \, j \in \{0, \ldots, k - 1\} \Big\},
    \end{align*}
    where $f \in \mathcal{L}_2(p(X))$, $g \in \mathcal{L}_2(p(Y))$ are square-integrable functions, and 
    \begin{align*}
        (f_k, g_k) \coloneqq \argmax &\Big\{\mathbb{E}[f(X)g(Y)]^2 \Big| \, \mathbb{E}[f(X)f_j(X)] = 0, \\
    &\phantom{=} \mathbb{E}[g(Y)g_j(A)] = 0, \, j \in \{0, \ldots, k - 1\} \Big\}.
    \end{align*}
    $\lambda_k(X; A)$ values are known as the {principal inertia components} (PICs) of the joint distribution, $p(X, A)$. The functions $f_k$ and $g_k$ are the {principal functions} of $X$ and $A$.
    \label{def:pic}
\end{definition}

The above definition presents a characterization of PICs as the construction of zero-mean, unit variance functions $f(X)$ and $g(A)$, which maximizes the correlation $\mathbb{E}[f(X)g(A)]$ without using any additional information. Note that alternative characterizations of PICs exist that include estimation errors and geometric interpretations~\citep{du2017principal}. 

From Definition~\ref{def:pic}, we observe that PICs satisfy: $\lambda_{k+1}(X; A) \leq \lambda_k(X; A) \leq 1$. PICs also correspond to the singular values of the joint probability matrix, specifically $\sqrt{\lambda_k(X; A)}$ is the $(k+1)$-st singular value of $\mathbf{Q}$, where $\mathbf{Q} = \mathbf{D}_X^{-1/2} \mathbf{P} \mathbf{D}_A^{-1/2}$. $\mathbf{D}_X \in \mathbb{R}^{m \times m}$ and $\mathbf{D}_A \in \mathbb{R}^{n \times n}$ are diagonal matrices with entries $[\mathbf{D}_X]_{i, i} = p_X(i)$ and $[\mathbf{D}_A]_{j, j} = p_A(j)$ respectively, where $m = |\mathcal{X}|$ and $n=|\mathcal{A}|$. $\mathbf{P} \in \mathbb{R}^{m \times n}$ is a matrix with entries $[\mathbf{P}]_{i, j}=p_{X, A}(i, j)$. 

\textbf{$k$-correlation}. Next, we introduce the notion of $k$-correlation and discuss how it relates to commonly known correlation metrics.

\begin{definition}[$k$-correlation] For random variables $X$, $A$ with finite support sets $\mathcal{X}, \mathcal{A}$ respectively, $k$-correlation between $X$ and $A$ is defined as:
\begin{equation}
    \mathcal{J}_k(X, A) \coloneqq \sum_{i=1}^k \lambda_k(X, A).
\end{equation}  
\end{definition}

The special cases of $k$-correlation correspond to well-known correlation metrics. For example, $\mathcal{J}_1(X, A) = \rho_m(X, A)^2$ is the R\'enyi maximum correlation~\citep{buja1990remarks}. Similarly, if $d = \min \{|\mathcal{X}|, |\mathcal{A}|\}-1$, then $\mathcal{J}_d(X, A) = \chi^2(X, A)$ is the chi-squared correlation.

Note that the smallest PIC is $\lambda_d(X; A)$ because $\mathbf{Q}$ has rank $(d+1)$. We will show how $\lambda_d(X, A)$ plays an important role in understanding the feasibility of perfect concept erasure. Next, we introduce the concept of erasure-utility coefficient for a given joint distribution, $p_{A, X}$.

\begin{definition}[Erasure-Utility Coefficient]
    The optimal erasure-utility coefficient for $p_{A,X}$ is:
    \begin{equation}
        v^* \coloneqq \inf \frac{I(Z; A)}{I(Z; X)}.\nonumber
    \end{equation}
    \label{def:erasure-util-coeff}
\end{definition}
We observe that the erasure-utility coefficient is closely related to concept erasure, where we want to reduce $I(Z; A)$ while maximizing $I(Z; X)$. Under perfect concept erasure, $v^* = 0$. Next, we restate the result from~\cite{du-etal-2020-general}, which provides an upper bound for the coefficient, $v^*$.

 \begin{restatable}[Optimal Erasure-Utility Coefficient~\citep{du2017principal}]{lemma}{erasurecoeff}
    Let $d \coloneqq \min\{|\mathcal A|, |\mathcal X|\} - 1$, and $\lambda_d(A, X)$  be the smallest principal inertia component of any distribution $p({A,X})$ with finite support $\mathcal A \times \mathcal{X}$. Then,
    \begin{equation}
        v^* \leq \begin{cases}
            \lambda_d(A, X), &\mathrm{ if } |\mathcal{X}| \leq |\mathcal{A}| \nonumber\\
            0, &\mathrm{otherwise}
        \end{cases}.
    \end{equation}
    \label{lem:erasure-coeff}
\end{restatable}

\feasibility*
\begin{proof}[Proof of Lemma~\ref{lem:feasibility}]
    From Lemma~\ref{lem:erasure-coeff} and Definition~\ref{def:erasure-util-coeff}, we can say that $v^*=0$ when either $\lambda_d(A, X) = 0$ or $|\mathcal{X}| > |\mathcal{A}|$. 
\end{proof}

\subsection{Proof of Lemma~\ref{lem:mi_result}}
\label{sec:lem1_proof}
\miresult*
\begin{proof}
The proof relies on the following chain rule for mutual information:
\begin{equation}
    I(Z; X) = I(Z; A) + I(Z; X|A) - I(Z; A|X)
\end{equation}
For completeness, we prove it below:
\begin{align}
    &I(Z; A|X) = H(Z|X) - H(Z|A, X)\nonumber\\
    \Rightarrow\; &I(Z; A|X) = H(Z) - H(Z) + H(Z|A) - H(Z|A) + H(Z|X) - H(Z|A, X)\nonumber\\
    \Rightarrow\; &I(Z; A|X) = \left[H(Z) - H(Z|A)\right] - \left[H(Z) - H(Z|X)\right] + \left[H(Z|A)   - H(Z|A, X)\right]\nonumber\\
    \Rightarrow\; &I(Z; A|X) = I(Z; A) - I(Z; X) + I(Z; X|A)\nonumber\\
    \Rightarrow\; &I(Z; X) = I(Z; A) - I(Z; A|X) + I(Z; X|A). \nonumber
\end{align}

In the above result, we plug-in $I(Z; A)=0$ (due to perfect erasure) and $I(Z; A|X)=0$ (due to Markov assumption~\ref{item:A1}). We obtain the following result:
\begin{align}
    &I(Z; X) = I(Z; X|A).
\end{align}

Now, replacing the upper bound of $I(Z; X|A) \leq H(X|A)$, we get:
\begin{align}
    I(Z; X) \leq H(X|A). \nonumber
\end{align}

The equality is satisfied when the following holds:
\begin{align*}
    &I(Z; X|A) = H(X|A) \nonumber\\
    \Rightarrow\; &H(X|A) - H(X|Z, A) = H(X|A) \nonumber\\
     \Rightarrow\; &H(X|Z, A) = 0. \nonumber\\
\end{align*}
This completes the proof.
\end{proof}

\subsection{Justification for Data Constraints in Definition~\ref{thm:data-cond}}
\label{sec:data-cond}

In this section, we provide the justification behind the choice of data constraints necessary for optimal erasure, where the distribution of representations should be permutations. 
First, we present an additional Lemma that would be useful in developing the intuition behind the data constraints. 
\begin{restatable}[Zero Conditional Entropy]{lemma}{zce}
    If $X$ and $Y$ are random variables defined on finite supports $\mathcal{X}$ and $\mathcal{Y}$ respectively and $Y=f(X)$, then $H(X|Y) = 0$ if and only if $f$ is a bijective map.
    \label{lem:zce}
\end{restatable}
\begin{proof}[Proof of Lemma~\ref{lem:zce}]
    First, we prove the forward direction, which says that $f$ is bijective given $H(X|Y)=0$.
    \begin{align}
        &H(X|Y) = 0 \nonumber\\
        \Rightarrow & \sum_{y \in \mathcal{Y}} p(y)H(X|Y=y) = 0\nonumber\\
        \xRightarrow{(a)} &\forall y \in \mathcal{Y}, H(X|Y=y) = 0 \label{eqn:zero-ent}
    \end{align}
    (a) holds because $\forall y \in \mathcal{Y},\; p(y)>0$. Eq.~\ref{eqn:zero-ent} implies that the probability distribution $\forall y, P(X|Y=y)$ is degenerate and there exists $x \in \mathcal{X}$ such that
    \begin{equation}
        \forall y, \exists x \text{ s.t. } P(X=x|Y=y) = 1, \forall x' \neq x, P(X=x'|Y=y) = 0.\label{eqn:prob}
    \end{equation} 
    Recall that we have $Y=f(X)$, and we want to show that $f$ is invertible (and therefore bijective). Suppose not, then  $\exists x, x' \in X$ that $f(x)=f(x')$. This contradicts Eq.~\ref{eqn:prob} because  $P(x|y)=1$. Therefore, $f$ must be invertible and bijective.

Next, we prove the backward direction, which shows that $H(X|Y)=0$ if $f$ is bijective. 

When $f$ is bijective, the following conditional distribution holds for a given $y$:
\begin{equation}
    P(X=f^{-1}(y)|Y=y) = 1 \text{ and } P(X=x|Y=y) = 1, \forall x\neq f^{-1}(y).\label{eqn:cond-dist}
\end{equation}
For a fixed $y$, using Eq.~\ref{eqn:cond-dist},  we can write the conditional entropy as:
\begin{align}
    H(X|Y=y) &= -\sum_{x \in \mathcal{X}} P(X=x|Y=y) \log P(X=x|Y=y) \nonumber\\
    &= \left[1\cdot\log 1 + \sum_{x \neq f^{-1}(y)} 0 \cdot \log 0\right] \nonumber\\
    &= 0.
\end{align}
Using the above result, we can write the following:
\begin{align}
    H(X|Y) = \sum_{y \in \mathcal{Y}} p(y)H(X|Y=y) = \sum_{y \in \mathcal{Y}} p(y) \cdot 0 = 0.
\end{align}
This completes the proof.
\end{proof}

With the above result, we are prepared to explain the data constraints in Definition~\ref{thm:data-cond}. 

\datacond*

The intuition of this definition is that we want to find an $f$ that makes use of the bijective maps between

First, we note that due to the disjoint support assumption~\ref{item:A4}, we can write $f$ as a piecewise function:
\begin{equation}
f(x) = \{f_i(x)|\;x\in \mathcal{X}_i\}, \text{ where } f_i: \mathcal{X}_i \rightarrow \mathcal{Z}.\label{eqn:piece}   
\end{equation}
    From Lemma~\ref{lem:mi_result},  we know that the utility outer bound holds when the following condition is true: 
    \begin{align}
        & H(X|Z, A) = 0 \nonumber\\
        \xRightarrow{(a)}\; &H(X|Z, A=a_i) = 0, \forall i \in \{1, \ldots, |\mathcal{A}|\}\label{eqn:1-1}
    \end{align}
    {(a) holds because $\forall i \in \{1,\dots,|\mathcal{A}|\}, p(a_i)>0$ by definition of the support. 
    Applying Lemma~\ref{lem:zce}, we observe that Eq.~\ref{eqn:1-1} is satisfied only when all $f_i$'s are bijective maps.}

    The bijective maps $f_i$'s imply that the support size of the input and output must be the same $|\mathcal{X}_i| = |\mathcal{Z}|$ (the support of $\mathcal{X}$ and $\mathcal{Z}$ is finite as stated in Assumption~\ref{item:A2}). This implies for every pair of concept group $\forall (i, j), |\mathcal{X}_i| =|\mathcal{X}_j|$, which proves the first data constraint.

    The bijective maps $f_i$'s also imply that $\forall i$ the distributions $P(Z|A=a_i)$ and $P(X|A=a_i)$ are permutations of each other, as shown below:
        \begin{align}
        \forall i, P[Z=z|A=a_i]
        &= \sum_{x \in \mathcal{X}_i} \mathbbm{1}\left[f_i(x)=z\right]P\left[X=x|A=a_i\right] \nonumber\\
        &= P_i\left[X=f_i^{-1}(z)\right]\nonumber\\
        &= P_i\left[X=\sigma_i(x)\right],\label{eqn:inp-out-perm}
    \end{align}
    where $\sigma_i(\cdot)$ is a permutation function and $P_i(X) = P(X|A=a_i)$.

    Next, we focus on the condition of perfect erasure:
    \begin{align}
         &I(Z; A) = 0\nonumber\\
         \Rightarrow\; &H(Z) - H(Z|A) = 0\nonumber\\
         \Rightarrow\; &H(Z) = H(Z|A) \nonumber\\
         \Rightarrow\; &H(Z) = H(Z|A=a_i), \forall i \in \{1, \ldots, |\mathcal{A}|\}.\label{eqn:entropy}
    \end{align}

    Next, we note that the output distribution can be written as: $P(Z) = \sum_i p(a_i)P(Z|A=a_i)$. We can write down the entropy of $Z$ as the entropy of the underlying distributions as below:
    \begin{align}
        H(Z) &= H\left(\sum_i p(a_i)P[Z|A=a_i]\right) \nonumber\\
        &\geq  \sum_i p(a_i)H\left(Z|A=a_i\right).\label{eqn:entropy-ineq}
    \end{align}
    Eq.~\ref{eqn:entropy-ineq} holds due to the concavity of the entropy function. However, for the Eq.~\ref{eqn:entropy} to hold, the equality in Eq.~\ref{eqn:entropy-ineq} needs to be satisfied. The equality is satisfied only when all the underlying distributions are equal, as shown below:
    \begin{align}
         &P[Z|A=a_i]=P[Z|A=a_j], \forall (i, j) \nonumber\\
         \Rightarrow\; &P_i[X=\sigma_i(x)]=P_j[X=\sigma_j(x)]\label{eqn:perm-eq}\\
         \Rightarrow\;&P_i[X=x]=P_j[X=\sigma_i^{-1}\sigma_j(x)]\label{eqn:inv}\\
         \Rightarrow\;&P_i[X=x]=P_j[X=\sigma_{ij}(x)],\nonumber
    \end{align}
where $\sigma_{ij} = \sigma_i \odot \sigma_j$. Eq.~\ref{eqn:perm-eq} holds due to the result in Eq.~\ref{eqn:inp-out-perm}. The result in Eq.~\ref{eqn:inv} implies that for all pairs of concept groups $(i, j)$, the underlying distributions are permutations of each other. This explains the second data constraint.

\subsection{Justification for Method~\ref{lem:equal-pef}}
\label{sec:equal-pef}

In this section, we will show that the perfect erasure function defined in Eq.~\ref{eqn:pef} achieves the MI outer bounds.

\equalpef*
First, we show that the perfect erasure function defined in Eq.~\ref{eqn:pef} achieves the utility outer bound. Using Lemma~\ref{lem:mi_result}, we get the following:
\begin{align}
    I(Z; X) &= I(Z; X|A) \nonumber\\
    &= \sum_i p(a_i)I(Z; X|A=a_i) \nonumber\\
    &= \sum_i p(a_i)[H(X|A=a_i) - H(X|Z, A=a_i)] \nonumber\\
    &= \sum_i p(a_i)H(X|A=a_i)\label{eqn:last}\\
    &= H(X|A).\nonumber
\end{align}
Eq.~\ref{eqn:last} can hold since $H(X|Z, A=a_i) = 0$ if $f_i: \mathcal{X}_i \rightarrow \mathcal{Z}$ is a bijective map.

Next, we show that the erasure function achieves perfect concept erasure. We compute the conditional distribution:
\begin{align}
    \forall z\in \mathcal{Z}, \;P(Z=z|A=a_i) &= \sum_{x \in \mathcal{X}_i} \mathbf{1}[f_i(x)=z]P(X=x|A=a_i)\label{eqn:bij}\\
    &= P_i(X=f_i^{-1}(z))\nonumber\\ 
    &= Q(z). \label{eqn:equal}
\end{align}
Eq.~\ref{eqn:bij} holds because $f_i$ is a bijective map. Eq.~\ref{eqn:equal} holds because of the erasure function definition in Eq.~\ref{eqn:perm_set}, where $P(X=x|A=a_i)=P_i(x) = Q(f_i(x))$. This implies the probability distribution over $\mathcal{Z}$ is:
\begin{align}
    \forall z\in \mathcal{Z},\; P(Z=z) &= \sum_i p(a_i)P[Z=z|A=a_i]\nonumber\\
    &= \sum_i p(a_i)Q(z) \nonumber\\ 
    &= Q(z). \label{eqn:eq_z}
\end{align}

Next, we use Eq.~\ref{eqn:equal} and Eq.~\ref{eqn:eq_z} to derive the privacy mutual information as shown below:
\begin{align}
    I(Z; A) &= H(Z) - H(Z|A) \nonumber\\
    &= H(Z) - \sum_i p(a_i)H(Z|A=a_i) \nonumber\\
    &= H(Q) - \sum_i p(a_i)H(Q) \label{eqn:q}\\
    &= H(Q) - H(Q) \nonumber\\
    &= 0.\nonumber
\end{align}
Eq.~\ref{eqn:q} holds because  $P(Z=z)=Q(z)$ and $\forall i \in \{1, \ldots, |\mathcal{A}|\}, \;P(Z=z|A=a_i)=Q(z)$.

The idea of this is to find the erasure function defined in Method~\ref{lem:equal-pef}, which achieves the outer bounds $I(Z; X) = H(X|A)$ and $I(Z; A) = 0$.

\section{METHOD}

\subsection{Optimization Procedure}
\label{appdx:opt}
In this section, we present an alternative method to optimize the objective function presented in Eq.~\ref{eqn:q-obj}, when the input distributions of the concept groups are unequal. This optimization is dependent on two sets of variables: common support $Q$ and coupling maps $\Gamma^{(i)} = \Gamma(P_i, Q)$. The key idea is to eliminate one set of variables and write the objective function only in terms of the coupling maps, since $Q = \sum_k \Gamma^{(i)}_{kj}$ is equal to the marginals of the coupling maps.

Therefore, we can write the objective as shown below:
\begin{align}
    &\max_{\{\Gamma^{(i)}\}_{i=1}^{|\mathcal{A}|}} \frac{1}{|\mathcal{A}|}H\left(\sum_{k}\Gamma^{(i)}_{kj}\right) - \sum_i p(a_i)H_{\Gamma^{(i)}}\left(P_i, \sum_{k}\Gamma^{(i)}_{kj}\right), \label{eq:pgd-obj}\\
    &\text{ where } \forall (i_1, i_2) \in [1, |\mathcal{A}|],\; \sum_{k}\Gamma^{(i_1)}_{kj} - \sum_{k}\Gamma^{(i_2)}_{kj} = 0 \text{ and } \sum_{k}\Gamma^{(i)}_{kj} = P_i.\label{eq:pgd-consts}
\end{align}
The above optimization can be solved using projected gradient descent~\citep{madry2017towards}, where we take unconstrainted gradient steps using the differentiable objective in Eq.~\ref{eq:pgd-obj} and then projecting the solutions obtained, $\{\Gamma_i\}$ to the linear constraints in Eq.~\ref{eq:pgd-consts}. However, since Eq.~\ref{eq:pgd-obj} is non-convex, solving this optimization still doesn't guarantee convergence to the global optima.

We applied this optimization method in small-scale synthetic setups and found it to function well. However, this method is not scalable to probability distributions with a large support set. This is because the number of linear constraints in Eq.~\ref{eq:pgd-consts} blows up, and the projection step is extremely expensive in practice. 

\clearpage
\section{EXPERIMENTS}
\label{sec:addl_exp}
\localtableofcontents

\subsection{Experimental Setup}
Performing concept erasure using {\X} does not require any parameter estimation except when the distributions are unequal. In this scenario, we use the Bayesian Optimization library~\citep{bayesopt} and use the upper confidence bound (UCB) of exploitation and exploration as the acquisition function (with $\kappa = 2.5$).  For evaluation of $\mathcal{V}$-information and MSE loss, we use the MLPClassifier and MLPRegressor functions from the scikit-learn~\citep{pedregosa2011scikit} library with its default settings. All the baseline models were trained on a 22GB NVIDIA Quadro RTX 6000 GPU and experiments were performed using the PyTorch~\citep{pytorch} framework. For baseline methods, we use the hyperparameters presented in the corresponding papers.

Next, we discuss the data generation process for the synthetic experiments reported in Section~\ref{sec:exp}. We consider two concept groups and sample $N$ 3D representations from a uniform distribution for each group. These representations form the support of the probability distributions we define next. We either define an equal (uniform or Gaussian) or unequal probability distribution using the support elements from each group. Using these distributions, we sample 10K representations from each group to form our final representation set. The concept label for each representation is the group from which it was sampled. For toxicity classification, we use the Jigsaw dataset, which contains online comments and the label is a fraction between 0 and 1 indicating the toxicity of the comment. We consider online comments associated with the following religion groups -- {Buddhist}, {Christian}, {Hindu}, {Jewish}, {Muslim}, and {others}.

\begin{figure}[h!]
    \centering
    \includegraphics[height=0.3\textwidth, keepaspectratio]{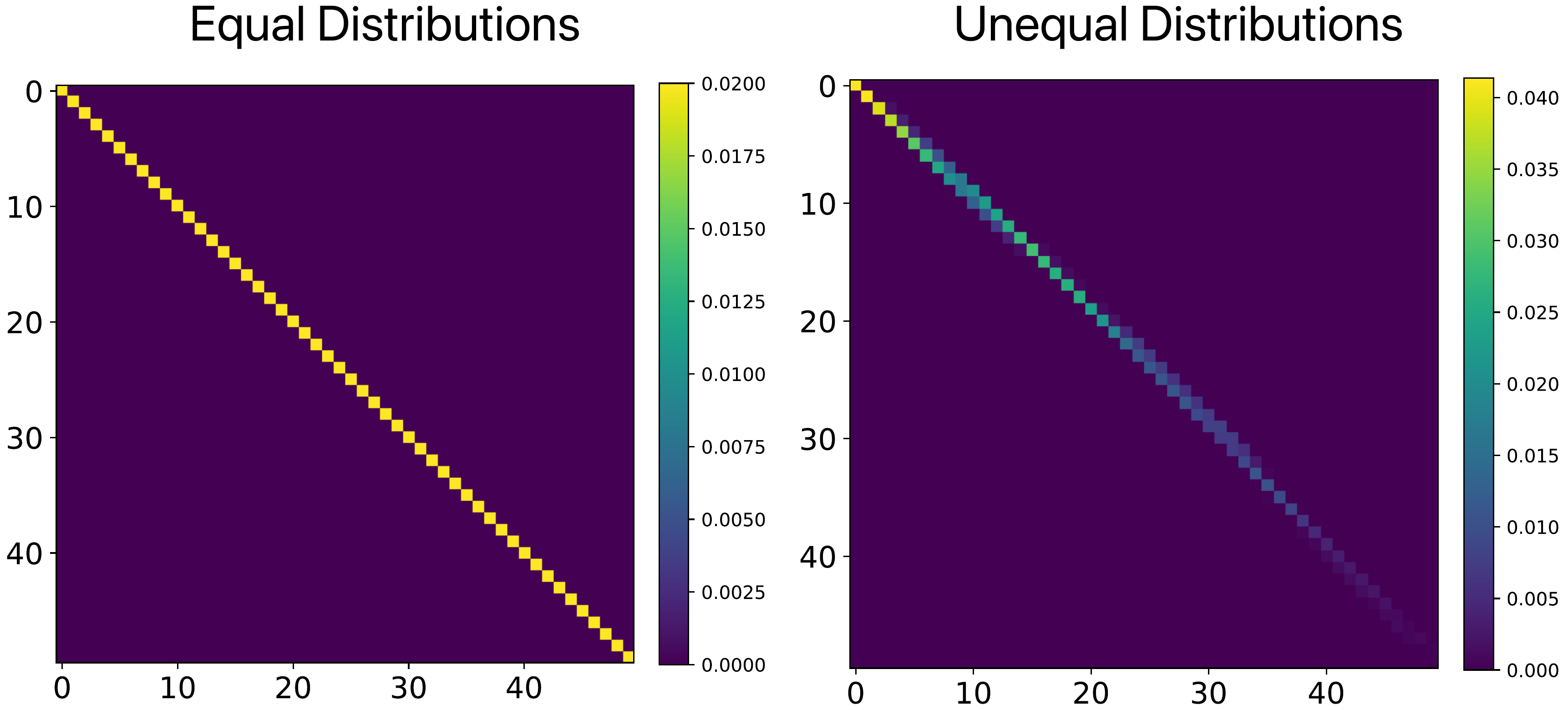}
    \caption{
    (\textit{Left}) We show the minimum entropy coupling (joint distribution) between two equal distributions. We observe that the coupling map is a diagonal matrix indicating a 1-1 bijective map. (\textit{Right}) We show that minimum entropy coupling between two unequal distributions is not a bijective map and individual support elements can be mapped to different elements. 
    }
    \label{fig:coupling}
\end{figure}

\subsection{Analysis Experiments}
We conduct analysis experiments to inspect the erasure functions obtained using the formulation in Eq.~\ref{eq:erasure-func}. Specifically, we visualize the coupling maps (joint distributions) formed between two distributions $\Gamma(P_i, Q)$. We consider two scenarios: (i) equal distributions ($P_i=Q$) and (ii) unequal distributions ($P_i \neq Q$). In Figure~\ref{fig:coupling}, we visualize the minimum entropy coupling maps formed in these two settings. In Figure~\ref{fig:coupling} (left), we observe that the coupling map is bijective (1-1 map), as predicted by Lemma~\ref{lem:equal-pef}. In Figure~\ref{fig:coupling} (center), we observe that the coupling map is not bijective and therefore perfect erasure is not achieved.

\begin{figure}[t!]
    \centering
    \includegraphics[width=0.8\linewidth]{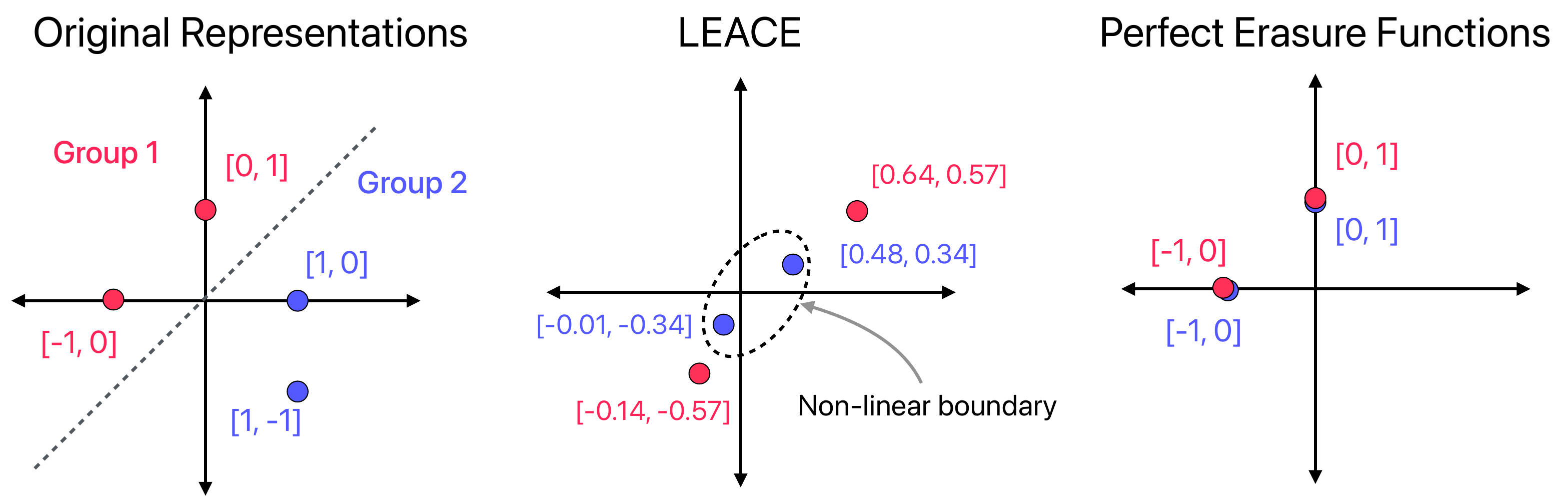}
    \caption{A toy example to illustrate the efficacy of erasure techniques LEACE and {\X}. (\textit{Left}) The original representation space with two linearly separable concept groups. (\textit{Middle}) The erased representations obtained using LEACE, where there still exists a non-linear boundary separating the two groups. (\textit{Right}) The erased representations from {\X} where the concept has been perfectly erased.}
    \label{fig:leace}
\end{figure}

\subsection{Comparison with LEACE}\label{sec:leace-vis}

In this section, we compare our proposed method {\X} with LEACE~\citep{belrose2024leace}. LEACE is another perfect erasure technique that prevents linear adversaries from extracting concept information from erased representations. However, stronger non-linear adversaries may still be able to extract concept information. We illustrate this using a toy example in Figure~\ref{fig:leace}. In this example, we consider a 2D original representation space (left plot in Figure~\ref{fig:leace}) with two concept groups (each with only two elements). After erasure using LEACE (Figure~\ref{fig:leace} (middle)), we observe that there still exists a non-linear boundary that can separate the two groups and thereby reveal concept information. In contrast, {\X} precisely maps elements between the two groups (Figure~\ref{fig:leace} (right)), making it impossible for any adversary to distinguish them. This example illustrates a scenario where LEACE can reveal concept information while {\X} remains effective.

\clearpage

\section{BROADER IMPACT}
\label{sec:broad_impacts}

Erasing sensitive concepts from data representations can reduce bias and enhance privacy, but it may also lead to the loss of valuable information, potentially diminishing the effectiveness of a machine learning model. The definition of sensitive concept attributes varies widely across cultural, ethical, and legal contexts. This work assumes that these attributes can be clearly defined and universally agreed upon, which is not always feasible. Therefore, developers should consider the societal impact carefully before implementing such erasure frameworks in practice.

{\X} is intended to be used in applications where the developer is aware of the
concept that needs to be erased. {\X} can only erase concepts where the labels are annotated
either as categorical attributes. One potential misuse of {\X} would be
to define relevant features for a task (e.g., educational background) as a concept to be erased.
In such scenarios, the decision-making system can end up relying on arbitrary personal information to make hiring decisions. It is important to have oversight over such practices by using standard fairness measures like demographic parity.

\end{document}